\documentclass[letterpaper]{article} 
\usepackage{meta/aaai2026}  
\usepackage{times}  
\usepackage{helvet}  
\usepackage{courier}  
\usepackage[hyphens]{url}  
\usepackage{graphicx} 
\usepackage{natbib}  
\usepackage{caption} 
\usepackage{algorithm}
\usepackage{algorithmic}

\usepackage{newfloat}
\usepackage{listings}
\DeclareCaptionStyle{ruled}{labelfont=normalfont,labelsep=colon,strut=off} 
\floatstyle{ruled}
\newfloat{listing}{tb}{lst}{}
\floatname{listing}{Listing}
\usepackage{graphicx}	
\usepackage{times}
\usepackage{caption}
\usepackage{placeins}
\usepackage{color, colortbl}
\usepackage{enumitem}
\usepackage{tabularx}
\usepackage{xstring}
\usepackage{multirow}
\usepackage{xspace}

\usepackage{booktabs}
\usepackage{amsmath,amssymb,amsfonts,amsthm}
\usepackage{array}
\usepackage{nicematrix}

\usepackage{tipa}
\usepackage{dsfont}
\usepackage{etoolbox}  

\graphicspath{sections/figs}

\definecolor{mygray}{RGB}{234,234,234}

\newcommand{\ra}[1]{\renewcommand{\arraystretch}{#1}}

\newcommand{\Ptr}{P_\mathsf{tr}}
\newcommand{\Pte}{P_\mathsf{te}}

\newcommand{\yh}{\hat{y}}

\DeclareMathOperator*{\argmin}{arg\,min}
\DeclareMathOperator*{\argmax}{arg\,max}

\newcommand{\ie}{\textit{i}.\textit{e}.}
\newcommand{\eg}{\textit{e}.\textit{g}.}

\usepackage{adjustbox}
\usepackage[table,dvipsnames]{xcolor}

\definecolor{darkgreen}{rgb}{0.13, 0.55, 0.13}
\newcommand{\improve}[1]{{\textbf{+#1}}}  

\newcommand{\smallimprove}[1]{\textsuperscript{$\uparrow$#1}}
\newcommand{\smalldecrease}[1]{\textsuperscript{$\downarrow$#1}}

\newcommand{\err}[1]{\textsuperscript{$\pm$#1}}

\definecolor{cvprblue}{rgb}{0.21,0.49,0.74}
\usepackage[pagebackref,breaklinks,colorlinks,allcolors=gray]{hyperref}  
\usepackage[capitalize]{cleveref}
\crefname{section}{Sec.}{Secs.}
\crefname{table}{Tab.}{Tabs.}
\crefname{figure}{Fig.}{Figs.}

\definecolor{lightblue}{rgb}{0.85, 0.93, 1.0}
\newcommand{\rot}[1]{\rotatebox{60}{#1}}
\usepackage{bm}
\usepackage{anyfontsize}  

\title{Your AI-Generated Image Detector Can Secretly Achieve SOTA Accuracy, If~Calibrated %
}
\author{
    Muli~Yang\textsuperscript{\rm 1},
    Gabriel~James Goenawan\textsuperscript{\rm 1}, 
    Henan~Wang\textsuperscript{{\rm 2}},
    Huaiyuan~Qin\textsuperscript{\rm 1},
    Chenghao~Xu\textsuperscript{\rm 3},
    Yanhua~Yang\textsuperscript{\rm 3}\thanks{Corresponding author.},
    Fen~Fang\textsuperscript{\rm 1}, 
    Ying~Sun\textsuperscript{\rm 1}, 
    Joo-Hwee~Lim\textsuperscript{\rm 1},
    Hongyuan~Zhu\textsuperscript{\rm 1}
}
\affiliations{
    \textsuperscript{\rm 1}Institute for Infocomm Research (I\textsuperscript{2}R), A*STAR, Singapore\\
    \textsuperscript{\rm 2}Independent Researcher\\
    \textsuperscript{\rm 3}Xidian University, China 

    \{yangml, goenawan, qinhy, fangf, suny, joohwee, zhuh\}@a-star.edu.sg; \\
    nnhhwang@gmail.com; 
    \{chx@stu., yanhyang@\}xidian.edu.cn
}

\begin{document}

\maketitle

\pagestyle{plain}
\thispagestyle{plain}

\begin{abstract}
  Despite being trained on balanced datasets, existing AI-generated image detectors often exhibit systematic bias at test time, frequently misclassifying fake images as real. 
  We hypothesize that this behavior stems from distributional shift in fake samples and implicit priors learned during training. Specifically, models tend to overfit to superficial artifacts that do not generalize well across different generation methods, leading to a misaligned decision threshold when faced with test-time distribution shift.
  To address this, we propose a theoretically grounded post-hoc calibration framework based on Bayesian decision theory. In particular, we introduce a learnable scalar correction to the model's logits, optimized on a small validation set from the target distribution while keeping the backbone frozen. This parametric adjustment compensates for distributional shift in model output, realigning the decision boundary even without requiring ground-truth labels. Experiments on challenging benchmarks show that our approach significantly improves robustness without retraining, offering a lightweight and principled solution for reliable and adaptive AI-generated image detection in the open world.
\end{abstract}

\begin{links}
    \link{Code}{https://github.com/muliyangm/AIGI-Det-Calib} 
\end{links}

\section{Introduction}\label{sec:intro}

\begin{figure*}[t]
  \centering
  \includegraphics[width=0.998\linewidth]{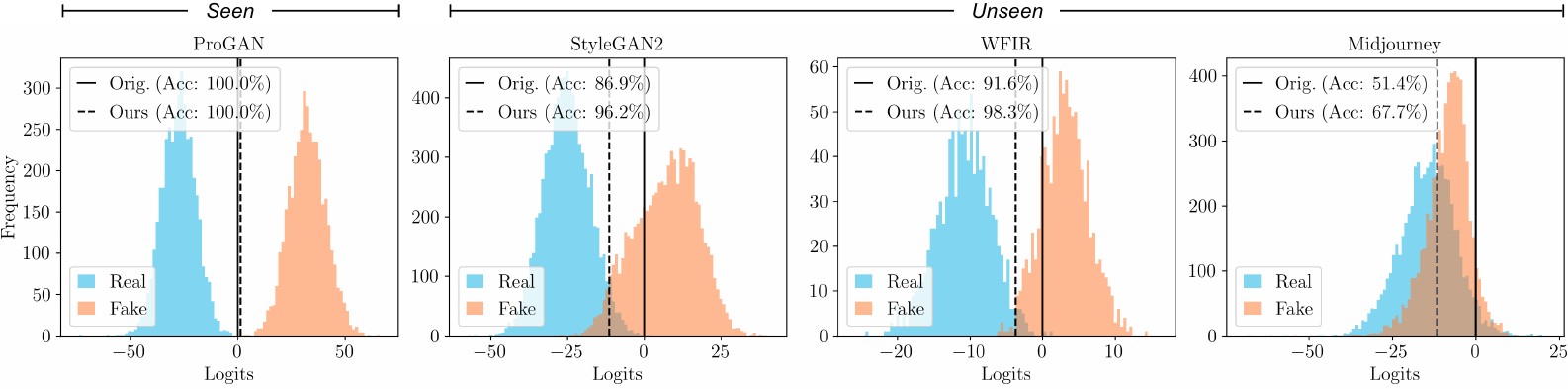}
  \caption{
 Logit distributions of a popular AI-generated image detector, CNNSpot~\cite{wang2020cnn}, pretrained on ProGAN-generated fake images and evaluated on previously unseen fake images from StyleGAN2, WhichFaceIsReal (WFIR), and Midjourney, reveal a tendency to misclassify these unfamiliar fake samples as real. Our proposed calibration method significantly enhances detection accuracy by adaptively shifting the decision boundary to better align with the skewed data distribution.
  }
  \label{fig:intro}
\end{figure*}

The rapid progress of AI-driven generative models has enabled the creation of highly realistic synthetic images.  Modern techniques, from generative adversarial networks (GANs) to diffusion models, now produce photographs and artwork that are often indistinguishable from real images. 
While these technologies empower creativity in fields such as art, design, and media production, they unintentionally introduce pressing challenges around authenticity, trust, and security. As synthetic media becomes increasingly accessible and indistinguishable from real content, developing reliable methods to detect AI-generated images is not only vital for digital forensics and intellectual property protection, but also foundational to maintaining information integrity in the age of generative AI~\cite{mahara2025methods}.

Most existing methods are developed and evaluated under constrained conditions, often assuming that test data shares the same distribution as the training set~\cite{zhu2024open,yang2025detecting}. In practice, however, detectors trained on forgeries from a specific generative model tend to overfit to superficial artifacts and fail to generalize when exposed to out-of-distribution (OOD) samples, such as those generated by novel architectures or exhibiting unseen statistical properties. A growing body of work has demonstrated that even state-of-the-art detectors, which perform nearly perfectly in-distribution, suffer dramatic performance degradation under distribution shift~\cite{nadimpalli2022improving,xiao2025generalizable,yan2025aide}, underscoring the brittleness of existing AI-generated image detectors in realistic and evolving generative environments.

In our empirical analysis, we identify a consistent failure pattern: even under class-balanced settings, models are significantly more likely to misclassify fake images as real. 
As illustrated in \cref{fig:intro}, we observe that the logits output by detectors on fake images exhibit a global shift, such that the optimal decision boundary no longer lies at zero. This deviation indicates a fundamental mismatch between the model’s learned decision threshold and the true distribution at test time, motivating a deeper investigation into its causes.
We hypothesize that this bias arises from a ``lazy'' decision mechanism developed during training~\cite{ojha2023towards,rajan2025stay}, where the model overly relies on superficial and spurious artifacts that are prevalent in seen fake samples, such as frequency noise or edge inconsistencies. When these artifacts are absent in unseen fake images, the model defaults to classifying them as real. This reliance on non-semantic fake-specific cues severely limits the model’s ability to generalize and leads to systematic under-detection of fakes at test time.
Further analysis reveals that this behavior can be attributed to two interacting sources of distributional shift: (a) \textit{label prior shift}, where the marginal distribution of real \textit{vs.}~fake images differs between training and testing, and (b)~\textit{class-conditional input shift}, where the distribution of inputs conditioned on the fake class changes due to new generation techniques. Notably, such shift is particularly prominent in the fake class, as different generative models exhibit coherent and systematic deviations in visual statistics, \eg, texture smoothness, spectral frequency, or semantic integrity. This causes the model’s estimated log-likelihood ratios to be consistently distorted, shifting the decision boundary towards the real class and exacerbating the tendency to misclassify fakes.

To formally model this effect, we adopt a Bayesian decision-theoretic framework and argue that, under distributional shift, the optimal classifier must adapt its decision boundary to the posterior induced by the new test distribution. Since retraining the full model is often impractical, we introduce a lightweight, theoretically grounded post-hoc calibration method: a learnable scalar bias $\alpha$ applied to the model’s output logits. This scalar globally adjusts the decision threshold, correcting for both label and input shift. Notably, $\alpha$ can be efficiently optimized using only a few unlabeled samples from the target distribution, while keeping the model backbone fixed. Under mild assumptions, this calibration approximates the Bayes-optimal classifier for the shifted distribution. As our approach requires no access to training data, loss functions, or model internals, as well as avoids retraining or additional supervision, it can be effortlessly applied to any AI-generated image detector, regardless of architecture.

In summary, our method provides a principled yet efficient way to restore model robustness under realistic distribution shift. Without modifying the original detector, our scalar calibration significantly improves performance across a wide range of generative scenarios. These results demonstrate that test-time decision bias, when properly diagnosed and corrected, can be mitigated with minimal computational and data overhead, offering a practical solution to the brittleness of modern AI-generated image detectors.

Our contributions are as follows:
\begin{itemize}
\item We identify a systematic bias of most existing AI-generated image detectors during test time, which frequently misclassify fake images as real.
\item Using Bayesian theory, we show that this systematic bias stems from the class-conditional input shift and label shift between mismatched train-test distributions.
\item Based on mild assumptions, we propose a post-hoc calibration method that optimizes a learnable scalar correction to the model's logits (while keeping the backbone frozen), largely improving most existing AI-generated image detectors' accuracy during test.
\end{itemize}

\section{Related Work}\label{sec:related}

\subsubsection{AI-Generated Image Detection.}
The rapid advancement of generative models, including GANs, VAEs, and diffusion models~\cite{ho2020denoising,xu2024rethinking}, has enabled the creation of highly photorealistic synthetic images, raising pressing concerns around misinformation and visual authenticity. This has led to an increasing demand for reliable methods to distinguish real from generated content~\cite{chen2024single,chen2024learning,nie2024detecting,zhong2025beyond,nguyen2025forensic,xiaohigh,jia2025secret,guillaro2025bias}.

A wide range of detection approaches has been proposed, leveraging pixel- or patch-level artifacts~\cite{nataraj2019detecting,chai2020makes,wang2020cnn,ju2022fusing, zhong2023patchcraft,lorenz2023detecting,tan2024npr,fupid}, 
and modeling generation-specific fingerprints through gradient features or network signatures~\cite{, marra2019gans, yu2019attributing,liu2022detecting,jeong2022fingerprintnet, tan2023learning, wang2023general,li2025revealing},
as well as utilizing reconstruction inconsistencies from pretrained generative models~\cite{wang2023dire,ricker2024aeroblade}. 
On the other hand, frequency-based analyses reveal spectral discrepancies among generated images~\cite{frank2020leveraging,dzanic2020fourier}, and 
AIDE~\cite{yan2025aide} further integrates spectral and semantic features for improved robustness.
Due to the prevalence of vision-language models (VLMs) and their wide applications~\cite{wang2023hierarchical,yang2025consistent,min2025vision}, 
cross-modal inconsistencies have also been explored using CLIP or other VLMs, 
enabling zero-shot or lightweight fake image detectors~\cite{ojha2023towards, liu2024forgery, cozzolino2024raising, koutlis2024leveraging}. 

While these approaches have demonstrated promising results, most assume static inference distributions and are vulnerable to domain shift, limiting their generalization to unseen generative models, which is the key challenge for real-world and practical deployment.

\subsubsection{Post-Hoc Calibration.}

Post-hoc calibration is an effective strategy for adapting model outputs to the test data distribution, which often differs from the training distribution. It has been extensively explored in contexts such as class-imbalanced learning~\cite{buda2018systematic,tang2020long,wu2021adversarial,hong2021disentangling}, where models tend to overfit to majority classes, and continual learning~\cite{wu2019large,hou2019learning,zhao2020maintaining}, where newer data is often favored. 
Notably, \citet{menon2020long} introduced \textit{logit adjustment}, a post-processing method that adjusts biased classifier outputs using theoretically optimal shifts based on class frequencies, offering a unifying framework for various heuristic bias-correction approaches~\cite{kang2019decoupling,ye2020identifying,islam2021class,kim2020adjusting}. 
Similar ideas have also been extended to mitigate prediction bias in foundation models~\cite{zhu2023generalized,mai2024fine}.

In this paper, we demonstrate that AI-generated image detectors exhibit significant prediction bias when the target data distribution shifts, even under class-balanced settings, which contrasts with previously studied scenarios. We provide a theoretical explanation grounded in Bayesian theories and propose a principled yet efficient {post-hoc calibration} approach that corrects for this bias using only a small number of labeled (or even unlabeled) target examples.

\section{Train-Test Misalignment in AI-Generated Image Detection}\label{sec:theory}

\newtheorem{theorem}{Theorem}
\newtheorem{lemma}{Lemma}
\newtheorem{proposition}{Proposition}

\subsection{A Probabilistic Problem Formulation}

We consider a binary classification problem for AI-generated image detection, where each input $x \in \mathcal{X}$ is assigned a binary label $y \in \{0,1\}$, with $y = 0$ denoting a real image and $y = 1$ a fake image. Let the joint distributions over data and labels in the training and test domains be denoted by $\Ptr(x, y)$ and $\Pte(x, y)$, respectively.

We assume that the classifier is trained to minimize cross-entropy under the training distribution $\Ptr$, and outputs a logit $f(x) \in \mathbb{R}$, whose sigmoid $\sigma(f(x)) = \frac{1}{1 + \exp(-f(x))}$ estimates the probability of the input being fake. The final decision is made by thresholding the probability at $\tau$, \ie, predicting $\yh = 1$ if $\sigma(f(x)) > \tau$, 
where $\tau$ denotes the decision threshold, typically set to $\tau = 0.5$ (\ie, when $f(x) = 0$).

However, in realistic open-world AI-generated image detection settings, the test distribution $\Pte$ often differs from the training distribution $\Ptr$. In particular, the fake images in the test set may be synthesized by different generative models not seen during training (\eg, training on StarGAN fakes, testing on Stable Diffusion fakes). This introduces a covariate mismatch in $P(x|y = 1)$, while $P(x|y = 0)$ (real images) remains approximately unchanged. Additionally, the class priors $\Ptr(y)$ and $\Pte(y)$ may also differ. Hence, we adopt a more expressive modeling assumption that jointly considers the following shifts:

\begin{itemize}
\item \textit{Class-Conditional Input Shift:} $\Ptr(x|y) \ne \Pte(x|y)$, especially for $y = 1$.
\item \textit{Label Prior Shift:} $\Ptr(y) \ne \Pte(y)$,
\end{itemize}

In practice, both shifts may occur simultaneously, leading to $\Ptr(x, y) \ne \Pte(x, y)$, and thus $\Pte(y|x)\neq \Ptr(y|x)$. In such cases, a model trained on $\Ptr$ may yield biased decision boundaries when evaluated on $\Pte$, as illustrated in \cref{fig:intro}.

\subsection{Default Threshold Is Not Bayes-Optimal}

We interpret the phenomenon shown in \cref{fig:intro} through the following propositions grounded in Bayesian decision theory~\cite{devroye1997probabilistic}.

\begin{proposition}[Bayes Non-optimality]\label{prop:optimal}
The default threshold $f(x) = 0$ (or $\tau = 0.5$) is not Bayes-optimal under class-conditional input shift and label prior shift.
\end{proposition}

\begin{proof}
By Bayes' theorem, the posterior under the training distribution is
$
\Ptr(y = 1 | x) = \frac{\Ptr(x \mid y = 1) \Ptr(y = 1)}{\Ptr(x)},
$
which gives the model output logit as
\begin{equation}
f(x) = \log\hspace{-2pt} \frac{\Ptr(y \hspace{-2pt}=\hspace{-2pt} 1 | x)}{\Ptr(y \hspace{-2pt}=\hspace{-2pt} 0 | x)} = \log\hspace{-2pt} \frac{\Ptr(x | y \hspace{-2pt}=\hspace{-2pt} 1) \Ptr(y \hspace{-2pt}=\hspace{-2pt} 1)}{\Ptr(x | y \hspace{-2pt}=\hspace{-2pt} 0) \Ptr(y \hspace{-2pt}=\hspace{-2pt} 0)}\,.
\end{equation}
Under the testing distribution $\Pte$, the Bayes-optimal classifier predicts $y = 1$ if
\begin{equation}
\begin{aligned}
\Pte(y=1|x)>\Pte(y=0|x)\,&\Longleftrightarrow\,\\
\log \frac{\Pte(x|1)}{\Pte(x|0)} &+ \log \frac{\Pte(1)}{\Pte(0)} > 0\,.
\end{aligned}
\end{equation}
The Bayes-optimal decision boundary satisfies
\begin{equation}\label{eq:decision_boundary}
\begin{aligned}
\log\frac{\Pte(x|1)}{\Pte(x|0)}+\log\frac{\Pte(1)}{\Pte(0)} &=\\
f(x)\hspace{-2pt}+\hspace{-2pt}\underbrace{\log\hspace{-2pt}\frac{\Pte(x|1)/\Ptr(x|1)}{\Pte(x|0)/\Ptr(x|0)}
\hspace{-2pt}+\hspace{-2pt}\log\hspace{-2pt}\frac{\Pte(1)/\Ptr(1)}{\Pte(0)/\Ptr(0)}}_{\Delta(x)}&=0\,.
\end{aligned}
\end{equation}
\cref{eq:decision_boundary} implies that the mismatch between $\Ptr$ and $\Pte$ causes the decision boundaries to differ by $\Delta(x)$:
\begin{align}\label{eq:delta}
\Delta(x) \doteq \log \frac{\Pte(x|1)}{\Ptr(x|1)} + \log \frac{\Pte(1)(1-\Ptr(1))}{\Ptr(1)(1-\Pte(1))}\,,
\end{align}
where we assume that the real image distribution remains stable between training and testing, \ie, $\Pte(x|0) \doteq \Ptr(x|0)$, and that the shift mainly occurs in the fake class. 

In \cref{eq:delta}, unless $\Delta(x) = 0$ holds for all $x$, which would require both covariate and label distributions to align, the Bayes-optimal decision boundary does not correspond to $f(x) = 0$ or $\tau(f(x)) = 0.5$. As a result, using $f(x) = 0$ as the default threshold leads to suboptimal decisions under distribution shift.
\end{proof}

\subsection{A Scalar Value Can Correct for the Threshold}

Let us analyze the behavior of $\Delta(x)$ over $x \sim \Pte(x|1)$, \ie, test samples that are truly fake.

\paragraph{Assumption 1 ({\normalfont Systematic Conditional Shift}).}
\textit{Let the log-likelihood ratio between test and train fake distributions be approximately constant:}
\begin{equation}
\log \frac{\Pte(x|1)}{\Ptr(x|1)} \doteq c, \quad \forall x \sim \Pte(x|1)\,.
\end{equation}

\noindent
This assumption is justified by the fact that fake images generated by different GANs or diffusion models tend to exhibit coherent and systematic deviations in visual features (\eg, frequency artifacts, textures)~\cite{ricker2022towards,you2025images}. 
As a result, the likelihood under the training fake distribution is consistently misaligned, causing the classifier to under- or over-estimate $f(x)$ in a fixed direction.

\paragraph{Assumption 2 ({\normalfont Consistent Prior Shift}).}
\textit{The prior over fake samples in the test set is shifted by a constant factor:}
\begin{equation}
\log \frac{\Pte(1)}{\Ptr(1)} = \delta\,.
\end{equation}

\noindent
This assumption is well justified by \citet{menon2020long},
and the rightmost term in \cref{eq:delta} can be written as $\delta'$, which is another constant that can be easily derived from $\delta$.

\begin{proposition}[Scalar Correction]\label{prop:scalar}
The suboptimal threshold can be corrected using a global additive scalar $\tilde{\alpha}$.
\end{proposition}

\begin{proof}
Under Assumptions 1 and 2, the model output bias $\Delta(x)$ in \cref{eq:delta} becomes
\begin{equation}
\Delta(x) \doteq c + \delta' = \mathrm{const}\,, \quad \forall x \sim \Pte(x|1)\,,
\end{equation}
which justifies the use of a global additive logit correction scalar,
$\tilde{\alpha} := -\Delta(x)$, 
for post-hoc calibration to correct for the biased decision boundary shown in \cref{prop:optimal}. 
The calibrated output logit is then written as
\begin{equation}\label{eq:alpha}
\tilde{f}(x) := f(x) - \tilde{\alpha}\,,\ \ \ \ \tilde{\alpha}=-(c + \delta')\,,
\end{equation}
where $\tilde{\alpha} \in \mathbb{R}$ manifests the optimal threshold adjustment value, which can be estimated using a small subset of data sampled from $\Pte$.
\end{proof}

Therefore, post-hoc calibration---namely, adapting $\tau$ to minimize expected test error---is not only justified but necessary for robust performance under distribution shift.

\section{Post-Hoc Calibration as a Remedy}\label{sec:method}

Building upon the probabilistic analysis in the previous section, we have established that two types of distribution shift between training and testing domains induce a systematic bias in the classifier's output logits. In particular, the logit function $f(x)$, learned under the training distribution $\Ptr(x, y)$, tends to systematically underestimate or overestimate the posterior probability $\Pte(y = 1 | x)$ when evaluated on the test distribution $\Pte(x, y)$. To address this issue, we propose a simple yet effective method: calibrating the logit output at test-time using a learnable additive bias $\alpha \in \mathbb{R}$.

We explore two variants of post-hoc calibration, depending on whether ground-truth labels from the test distribution are available. Both variants leverage kernel density estimation (KDE)~\cite{rosenblatt1956remarks,parzen1962estimation} as their underlying mechanism for its simplicity and effectiveness.

\subsubsection{Supervised Calibration.}

In scenarios where a small amount of labeled target data is available, we propose a supervised logit calibration method based on KDE and accuracy maximization. The core idea is to model the distribution of classifier logits for each class using KDE, and to select a calibration value that maximizes the classification accuracy on the target distribution.

Let $z=f(x) \in \mathbb{R}$ denote the real-valued logits produced by a binary classifier.
Given a small set of labeled samples from the target domain, we estimate two class-conditional densities using Gaussian kernel density estimation:
\begin{align}
p_0(z) := p(z | y = 0)\,, \quad p_1(z) := p(z | y = 1)\,.
\end{align} 
As shown in \cref{fig:kde}, the expected classification error for a given calibration value $\alpha$ is then defined as
\begin{equation}
\begin{aligned}
\mathcal{R}(\alpha) &= P(z > \alpha \mid y = 0) + P(z \leqslant \alpha \mid y = 1) \\
&=  \int_{-\infty}^\alpha p_1(z) \, dz + \int_{\infty}^\alpha p_0(z) \, dz\,.
\end{aligned}
\end{equation}
The optimal calibration is thus found by minimizing $\mathcal{R}(\alpha)$:
\begin{align}
    \alpha^\star = \arg\min_\alpha \mathcal{R}(\alpha) \,,
\end{align}
where we leave the detailed method description and optimization procedure to Appendix.

This method is both simple and effective, requiring only a handful of labeled examples and no assumptions on the parametric form of the underlying logit distributions. Its ability to adaptively model asymmetric or multimodal distributions makes it especially useful in real-world transfer and domain adaptation scenarios.
In \cref{sec:different_methods} we also compare it with several other supervised threshold selection methods.

\begin{figure}[t]
  \centering
  \includegraphics[width=0.96\linewidth]{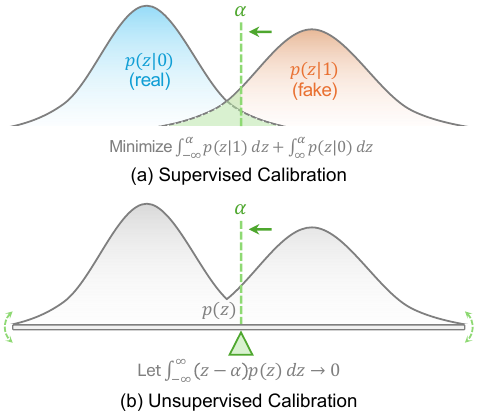}
  \caption{
  Conceptual illustration of our proposed (a) supervised and (b) unsupervised calibration methods, both designed to identify an optimal scalar $\alpha$ that achieves an ideal separation between real and fake distributions, with or without access to ground-truth labels.
  }
  \label{fig:kde}
\end{figure}

\begin{table*}[t]
\centering
\ra{1.0}
\setlength\tabcolsep{5pt} 
\resizebox{\linewidth}{!}
{
\fontsize{13}{13}\selectfont 
\begin{tabular}{@{}lcccccccccccccccccc@{}}
\toprule
Method $\downarrow$  &\rot{ProGAN} &\rot{StyleGAN} & \rot{BigGAN} & \rot{CycleGAN} &\rot{StarGAN} &\rot{GauGAN} &\rot{StyleGAN2} &\rot{WFIR} &\rot{ADM} &\rot{Glide} & \rot{{Midjourney}} & \rot{{SD v1.4}} & \rot{{SD v1.5}}& \rot{{VQDM}}& \rot{{Wukong}}& \rot{{DALLE2}} & {\textbf{Average}} &$\Delta$\\ \midrule

CNNSpot~\shortcite{wang2020cnn} & 100.00 &90.17 &71.17 &87.62 &94.60 & 81.43 &86.91 &91.65 &60.40 &58.06 & 51.39 &50.57 &50.53 &56.45 &51.02 &51.25 &70.83\phantom{\err{0.00}}\cellcolor{lightblue}  \\
+ \textit{Ours (Sup.)} &100.00	&97.95	&83.30	&88.64	&95.02	&86.83	&96.15	&98.30	&67.94	&65.63	&67.48&	64.64	&65.00	&63.12	&59.75	&51.70	&78.22\err{0.12}\cellcolor{lightblue}  &\improve{7.39}\\
+ \textit{Ours (Unsup.)} &100.00	&97.59	&83.25	&88.80	&95.15	&86.70	&95.41	&98.10	&67.42	&65.56	&67.61	&63.98	&64.42	&62.68	&57.90	&51.90	&77.90\err{0.08}\cellcolor{lightblue} &\improve{7.08}\\
\midrule

FreDect~\shortcite{frank2020leveraging} &99.36	&78.03	&81.97	&78.77	&94.62	&80.56	&66.19	&50.75	&63.40	&54.12	&45.87	&38.77	&39.21	&77.80	&40.29	&34.70	&64.03\phantom{\err{0.00}}\cellcolor{lightblue} \\
+ \textit{Ours (Sup.)} &99.70	&77.78	&86.48	&80.77	&97.40	&80.21	&74.92	&58.00	&68.19	&63.19	&45.99	&49.52	&42.43	&78.05	&49.62	&49.75	&68.88{\err{0.48}}\cellcolor{lightblue} &\improve{4.85}\\
+ \textit{Ours (Unsup.)} &99.74	&73.49	&86.10	&81.19	&96.85	&80.45	&74.04	&57.70	&66.55	&61.44	&43.93	&34.66	&34.29	&77.98	&38.32	&34.65	&65.09{\err{0.12}}\cellcolor{lightblue} &\improve{1.06}\\
\midrule

Fusing~\shortcite{ju2022fusing} &99.99	&85.21	&77.35	&87.02	&97.02	&76.92	&83.27	&66.85	&56.53	&57.16	&52.17	&51.05	&51.36	&55.13	&51.72	&52.85	&68.85\phantom{\err{0.00}}\cellcolor{lightblue} &\\
+ \textit{Ours (Sup.)} &99.99	&96.28	&84.62	&89.02	&98.05	&84.36	&96.44	&87.05	&67.53	&69.92	&65.20	&62.42	&63.21	&68.83	&60.11	&61.75	&78.42{\err{0.11}}\cellcolor{lightblue} &\improve{9.57}\\
+ \textit{Ours (Unsup.)} &99.99	&89.58	&81.97	&88.64	&97.52	&80.49	&88.26	&82.80	&65.43	&67.44	&61.88	&59.86	&60.23	&65.97	&58.25	&61.30	&75.60{\err{0.10}}\cellcolor{lightblue} &\improve{6.75}\\
\midrule

LNP~\shortcite{liu2022detecting} &99.78	&92.15	&83.05	&84.60	&99.90	&75.17	&93.87	&55.15	&78.98	&79.67	&57.83	&79.37	&79.24	&69.94	&75.74	&93.05	&81.09\phantom{\err{0.00}}\cellcolor{lightblue} &\\
+ \textit{Ours (Sup.)} &99.81	&94.48	&84.75	&88.00	&100.00	&77.21	&97.25	&91.40	&85.38	&81.89	&79.77	&87.55	&86.86	&80.53	&85.56	&94.30	&88.42{\err{0.12}}\cellcolor{lightblue} &\improve{7.33}\\
+ \textit{Ours (Unsup.)} &99.52	&94.20	&83.65	&86.64	&99.90	&76.88	&96.82	&91.00	&84.14	&78.21	&78.33	&83.41	&83.86	&79.89	&83.86	&94.35	&87.17{\err{0.34}}\cellcolor{lightblue} &\improve{6.07}\\
\midrule

LGrad~\shortcite{tan2023learning} &99.98	&90.47	&88.92	&85.69	&99.62	&82.81	&87.77	&58.50	&66.13	&71.67	&70.47	&65.24	&65.91	&74.51	&60.30	&71.25	&77.45\phantom{\err{0.00}}\cellcolor{lightblue} &\\
+ \textit{Ours (Sup.)} &99.99	&94.23	&90.18	&86.11	&99.85	&86.96	&95.55	&59.70	&66.07	&75.08	&70.23	&66.29	&67.26	&74.31	&62.81	&78.40	&79.56{\err{0.18}}\cellcolor{lightblue} &\improve{2.11}\\
+ \textit{Ours (Unsup.)} &99.95	&94.16	&90.05	&86.11	&99.65	&86.74	&95.61	&58.20	&66.00	&75.24	&70.41	&66.07	&66.83	&74.17	&62.07	&78.35	&79.35{\err{0.08}}\cellcolor{lightblue} &\improve{1.90}\\
\midrule

UnivFD~\shortcite{ojha2023towards}  &99.81	&84.93	&95.08	&98.33	&95.75	&99.47	&74.96	&86.90	&66.87	&62.46	&56.13	&63.66	&63.49	&85.31	&70.93	&50.75	&78.43\phantom{\err{0.00}}\cellcolor{lightblue} &\\
+ \textit{Ours (Sup.)} &99.72	&90.09	&95.70	&98.33	&96.62	&99.48	&92.60	&90.15	&78.38	&76.96	&68.66	&79.83	&79.17	&89.81	&83.08	&63.00	&86.35{\err{0.14}}\cellcolor{lightblue} &\improve{7.92}\\
+ \textit{Ours (Unsup.)}  &99.74	&89.53	&95.33	&98.30	&96.45	&99.50	&92.22	&90.30	&78.38	&76.89	&68.77	&79.59	&79.09	&89.16	&83.11	&62.80	&86.20{\err{0.09}}\cellcolor{lightblue}&\improve{7.77}\\
\midrule

RINE~\shortcite{koutlis2024leveraging} &100.00	&88.86	&99.60	&99.32	&99.55	&99.77	&94.50	&97.35	&74.61	&80.72	&57.12	&83.96	&83.35	&89.79	&84.95	&54.85	&86.77\phantom{\err{0.00}}\cellcolor{lightblue} &\\
+ \textit{Ours (Sup.)} &99.96	&95.73	&99.62	&99.74	&99.82	&99.87	&99.31	&97.15	&89.40	&92.95	&80.48	&93.90	&93.96	&95.66	&93.97	&81.50	&94.56{\err{0.07}}\cellcolor{lightblue} &\improve{7.80}\\
+ \textit{Ours (Unsup.)} &99.99	&93.57	&98.47	&99.74	&97.20	&99.82	&99.01	&95.90	&88.92	&91.35	&80.16	&92.88	&92.75	&95.88	&93.40	&81.60	&93.79{\err{0.22}}\cellcolor{lightblue} &\improve{7.02}\\
\midrule

AIDE~\shortcite{yan2025aide} &99.99	&99.65	&83.95	&98.49	&99.90	&73.24	&98.01	&94.75	&93.47	&95.09	&77.20	&92.95	&92.84	&95.12	&93.49	&96.60	&92.80\phantom{\err{0.00}}\cellcolor{lightblue} &\\
+ \textit{Ours (Sup.)} &99.99	&99.75	&85.55	&98.60	&99.92	&91.01	&99.52	&95.65	&94.20	&95.08	&79.05	&93.14	&93.19	&95.62	&93.80	&96.60	&94.42{\err{0.13}}\cellcolor{lightblue} &\improve{1.62}\\
+ \textit{Ours (Unsup.)} &99.98	&99.75	&85.30	&97.31	&99.77	&82.17	&99.30	&95.60	&94.08	&94.95	&78.64	&93.37	&93.16	&95.54	&93.64	&96.50	&93.69{\err{0.08}}\cellcolor{lightblue} &\improve{0.90}\\
\midrule

Effort~\shortcite{yan2024effort} &99.92	&89.17	&99.12	&99.96	&100.00	&99.94	&90.21	&70.30	&59.18	&70.43	&50.58	&71.62	&71.42	&75.13	&69.34	&53.25	&79.35\phantom{\err{0.00}}\cellcolor{lightblue} &\\
+ \textit{Ours (Sup.)} &99.89	&94.02	&99.45	&100.00	&100.00	&99.94	&96.02	&93.05	&76.78	&86.58	&68.63	&88.88	&88.55	&89.91	&87.17	&62.75	&89.48{\err{0.12}}\cellcolor{lightblue} &\improve{10.13}\\
+ \textit{Ours (Unsup.)} &99.92	&90.22	&99.12	&99.96	&100.00	&99.94	&91.68	&82.20	&71.62	&78.25	&68.62	&79.86	&79.79	&79.10	&78.64	&62.35	&85.08{\err{0.21}}\cellcolor{lightblue} &\improve{5.73}\\

\bottomrule
\end{tabular}
}
\caption{
{Results on AIGCDetectBenchmark~\cite{zhong2023patchcraft}.} 
Accuracies (\%) of different detectors (rows) in detecting real and fake images from different generators (columns) are reported. These methods are trained on real images from LSUN and fake images generated by ProGAN and evaluated over 16 generators.
}
\label{tab:aigcdetect}
\end{table*}

\begin{table*}[t]
\centering
\ra{0.5}
\setlength\tabcolsep{5pt} 
\resizebox{1\linewidth}{!}
{%
\begin{tabular}{@{}lcccccccccc@{}}
\toprule
Method $\downarrow$  &{Midjourney} &{SD v1.4} & {SD v1.5} & {ADM} &{GLIDE} &{Wukong} &{VQDM} &{BigGAN} & {\textbf{Average}} & {$\Delta$}\\ \midrule

CNNSpot~\shortcite{wang2020cnn} &55.13\phantom{\err{0.00}}	&99.96\phantom{\err{0.00}}	&99.85\phantom{\err{0.00}}	&50.11\phantom{\err{0.00}}	&50.47\phantom{\err{0.00}}	&99.28\phantom{\err{0.00}}	&50.18\phantom{\err{0.00}}	&49.98\phantom{\err{0.00}}	&69.37\phantom{\err{0.00}}\cellcolor{lightblue}  \\
+ \textit{Ours (Sup.)} &76.12\err{0.11}	&99.95\err{0.05}	&99.85\err{0.13}	&54.39\err{0.35}	&64.98\err{0.52}	&99.72\err{0.22}	&51.20\err{0.87}	&49.08\err{0.49}	&74.41{\err{0.11}}\cellcolor{lightblue}	&\improve{5.04}\\
+ \textit{Ours (Unsup.)} &75.92\err{0.52}	&99.94\err{0.05}	&99.82\err{0.18}	&49.06\err{1.02}	&64.51\err{0.38}	&99.51\err{0.19}	&40.08\err{1.22}	&38.07\err{0.12}	&70.86{\err{0.21}}\cellcolor{lightblue}	&\improve{1.49}\\
\midrule

FreDect~\shortcite{frank2020leveraging} &61.34\phantom{\err{0.00}}	&99.60\phantom{\err{0.00}}	&99.46\phantom{\err{0.00}}	&50.40\phantom{\err{0.00}}	&54.89\phantom{\err{0.00}}	&97.13\phantom{\err{0.00}}	&49.87\phantom{\err{0.00}}	&50.65\phantom{\err{0.00}}	&70.42\phantom{\err{0.00}}\cellcolor{lightblue} \\
+ \textit{Ours (Sup.)} 					&81.67\err{0.27}	&99.62\err{0.13}	&99.48\err{0.20}	&50.57\err{0.31}	&75.29\err{0.59}	&98.01\err{0.51}	&49.40\err{0.64}	&55.90\err{1.37}	&76.24{\err{0.21}}\cellcolor{lightblue}&	\improve{5.83}\\
+ \textit{Ours (Unsup.)} 				&81.67\err{0.33}	&99.60\err{0.02}	&99.45\err{0.04}	&39.21\err{0.78}	&75.24\err{0.48}	&97.97\err{0.12}	&34.86\err{0.23}	&55.53\err{0.44}	&72.94{\err{0.11}}\cellcolor{lightblue}&	\improve{2.52}\\
\midrule

Fusing~\shortcite{ju2022fusing} 		&58.72\phantom{\err{0.00}}	&99.98\phantom{\err{0.00}}	&99.91\phantom{\err{0.00}}	&57.05\phantom{\err{0.00}}	&73.50\phantom{\err{0.00}}	&99.96\phantom{\err{0.00}}	&64.75\phantom{\err{0.00}}	&55.43\phantom{\err{0.00}}	&76.16\phantom{\err{0.00}}\cellcolor{lightblue}\\
+ \textit{Ours (Sup.)} 					&82.56\err{1.50}	&99.99\err{0.15}	&99.89\err{0.06}	&86.31\err{2.67}	&96.29\err{0.66}	&99.90\err{0.61}	&94.44\err{0.97}	&62.18\err{0.76}	&90.20{\err{0.51}}\cellcolor{lightblue}&	\improve{14.03}\\
+ \textit{Ours (Unsup.)} 				&66.09\err{1.83}	&99.98\err{0.00}	&99.91\err{0.00}	&67.18\err{1.23}	&80.42\err{0.85}	&99.96\err{0.00}	&73.80\err{0.89}	&60.17\err{0.78}	&80.94{\err{0.45}}\cellcolor{lightblue}&	\improve{4.78}\\
\midrule

LNP~\shortcite{liu2022detecting} 		&50.34\phantom{\err{0.00}}	&99.93\phantom{\err{0.00}}	&99.86\phantom{\err{0.00}}	&60.88\phantom{\err{0.00}}	&50.10\phantom{\err{0.00}}	&99.78\phantom{\err{0.00}}	&75.44\phantom{\err{0.00}}	&62.02\phantom{\err{0.00}}	&74.79\phantom{\err{0.00}}\cellcolor{lightblue}\\
+ \textit{Ours (Sup.)} 					&58.34\err{0.93}	&99.91\err{0.05}	&99.85\err{0.06}	&94.75\err{0.18}	&85.69\err{0.12}	&99.77\err{0.11}	&96.38\err{0.13}	&61.95\err{0.32}	&87.08{\err{0.09}}\cellcolor{lightblue}&	\improve{12.29}\\
+ \textit{Ours (Unsup.)} 				&54.44\err{1.56}	&99.97\err{0.03}	&99.83\err{0.03}	&94.35\err{0.57}	&85.83\err{0.06}	&99.78\err{0.06}	&96.37\err{0.22}	&62.00\err{0.64}	&86.57{\err{0.25}}\cellcolor{lightblue}&	\improve{11.78}\\
\midrule

LGrad~\shortcite{tan2023learning} 		&68.70\phantom{\err{0.00}}	&99.40\phantom{\err{0.00}}	&99.41\phantom{\err{0.00}}	&51.90\phantom{\err{0.00}}	&61.57\phantom{\err{0.00}}	&96.97\phantom{\err{0.00}}	&51.04\phantom{\err{0.00}}	&49.73\phantom{\err{0.00}}	&72.34\phantom{\err{0.00}}\cellcolor{lightblue}\\
+ \textit{Ours (Sup.)} 					&81.33\err{0.22}	&99.27\err{0.06}	&99.41\err{0.07}	&54.07\err{1.75}	&76.64\err{0.33}	&97.82\err{0.14}	&56.56\err{0.58}	&52.38\err{1.86}	&77.19{\err{0.43}}\cellcolor{lightblue}&	\improve{4.85}\\
+ \textit{Ours (Unsup.)} 				&80.33\err{0.64}	&99.24\err{0.07}	&99.39\err{0.08}	&51.76\err{0.61}	&76.78\err{0.19}	&97.72\err{0.19}	&56.53\err{0.13}	&46.65\err{0.35}	&76.05{\err{0.13}}\cellcolor{lightblue}&	\improve{3.71}\\
\midrule

UnivFD~\shortcite{ojha2023towards}  	&85.16\phantom{\err{0.00}}	&96.89\phantom{\err{0.00}}	&96.74\phantom{\err{0.00}}	&53.86\phantom{\err{0.00}}	&73.61\phantom{\err{0.00}}	&92.17\phantom{\err{0.00}}	&56.56\phantom{\err{0.00}}	&61.40\phantom{\err{0.00}}	&77.05\phantom{\err{0.00}}\cellcolor{lightblue}\\
+ \textit{Ours (Sup.)} 					&87.82\err{0.37}	&96.89\err{0.38}	&96.67\err{0.20}	&57.67\err{0.92}	&81.23\err{0.07}	&92.67\err{0.26}	&66.94\err{0.66}	&79.05\err{0.43}	&82.37{\err{0.18}}\cellcolor{lightblue}&	\improve{5.32}\\
+ \textit{Ours (Unsup.)}  				&87.71\err{0.24}	&96.77\err{0.38}	&96.49\err{0.52}	&56.62\err{0.12}	&81.16\err{0.17}	&92.52\err{0.37}	&66.85\err{0.06}	&78.97\err{0.20}	&82.14{\err{0.08}}\cellcolor{lightblue}&	\improve{5.09}\\
\midrule

RINE~\shortcite{koutlis2024leveraging} 	&69.38\phantom{\err{0.00}}	&99.97\phantom{\err{0.00}}	&99.90\phantom{\err{0.00}}	&56.73\phantom{\err{0.00}}	&53.24\phantom{\err{0.00}}	&99.95\phantom{\err{0.00}}	&88.95\phantom{\err{0.00}}	&86.15\phantom{\err{0.00}}	&81.78\phantom{\err{0.00}}\cellcolor{lightblue}\\
+ \textit{Ours (Sup.)} 					&96.33\err{0.13}	&99.98\err{0.02}	&99.91\err{0.10}	&92.86\err{0.31}	&97.34\err{0.26}	&99.95\err{0.03}	&98.70\err{0.11}	&98.45\err{0.29}	&97.94{\err{0.08}}\cellcolor{lightblue}&	\improve{16.16}\\
+ \textit{Ours (Unsup.)} 				&96.06\err{0.29}	&99.97\err{0.00}	&99.86\err{0.02}	&92.73\err{0.15}	&95.67\err{0.59}	&99.88\err{0.07}	&98.63\err{0.20}	&97.90\err{0.51}  &97.59{\err{0.12}}\cellcolor{lightblue}&	\improve{15.80}\\
\midrule

AIDE~\shortcite{yan2025aide} 			&79.38\phantom{\err{0.00}}	&99.74\phantom{\err{0.00}}	&99.75\phantom{\err{0.00}}	&78.54\phantom{\err{0.00}}	&91.80\phantom{\err{0.00}}	&98.67\phantom{\err{0.00}}	&80.27\phantom{\err{0.00}}	&77.20\phantom{\err{0.00}}  &88.17\phantom{\err{0.00}}\cellcolor{lightblue}\\
+ \textit{Ours (Sup.)} 					&91.66\err{0.48}	&99.74\err{0.14}	&99.75\err{0.17}	&89.81\err{0.25}	&96.97\err{0.29}	&99.22\err{0.41}	&91.63\err{0.61}	&77.22\err{0.48}	&93.25{\err{0.13}}\cellcolor{lightblue}&	\improve{5.08}\\
+ \textit{Ours (Unsup.)} 				&86.29\err{0.84}	&99.78\err{0.02}	&99.74\err{0.02}	&85.62\err{0.43}	&94.88\err{0.39}	&98.98\err{0.15}	&86.94\err{0.94}	&76.60\err{0.13}	&91.10{\err{0.20}}\cellcolor{lightblue}&	\improve{2.94}\\
\midrule

Effort~\shortcite{yan2024effort} 		&82.40\phantom{\err{0.00}}	&99.83\phantom{\err{0.00}}	&99.81\phantom{\err{0.00}}	&78.78\phantom{\err{0.00}}	&93.31\phantom{\err{0.00}}	&97.42\phantom{\err{0.00}}	&91.70\phantom{\err{0.00}}	&88.05\phantom{\err{0.00}}	&91.41\phantom{\err{0.00}}\cellcolor{lightblue}\\
+ \textit{Ours (Sup.)} 					&94.09\err{0.43}	&99.82\err{0.15}	&99.83\err{0.08}	&96.57\err{0.16}	&98.38\err{0.03}	&98.47\err{0.12}	&97.10\err{0.18}	&88.85\err{0.14}	&96.64{\err{0.09}}\cellcolor{lightblue}&	\improve{5.23}\\
+ \textit{Ours (Unsup.)} 				&93.40\err{0.36}	&99.71\err{0.07}	&99.70\err{0.06}	&96.57\err{0.08}	&98.31\err{0.13}	&98.47\err{0.08}	&96.88\err{0.20}	&87.58\err{0.84}	&96.33{\err{0.16}}\cellcolor{lightblue}&	\improve{4.92}\\

\bottomrule
\end{tabular}
}
\caption{{Results on GenImage~\cite{zhu2023genimage}}. Accuracies (\%) of different detectors (rows) in detecting real and fake images from different generators (columns) are reported. These methods are trained on 
real images from ImageNet and fake images generated by SD v1.4 and evaluated over 8 generators.
}
\label{tab:genimage}
\end{table*}

\begin{table}[t]
\centering
\ra{0.5}
\setlength\tabcolsep{1.5pt} 
\resizebox{1\linewidth}{!}{
\begin{tabular}{@{}lccccc@{}}
    
\toprule
& Original& \multicolumn{2}{c}{{JPEG Compression}}  & \multicolumn{2}{c@{}}{{Gaussian Blur}}  \\
\cmidrule(lr){3-4} \cmidrule(l){5-6}
Method $\downarrow$& {Average} & {$\text{QF}\hspace{-1pt}=\hspace{-1pt}95$} & {$\text{QF}\hspace{-1pt}=\hspace{-1pt}90$} & {$\sigma\hspace{-1pt}=\hspace{-1pt}1.0$} & {$\sigma\hspace{-1pt}=\hspace{-1pt}2.0$} \\ 
\midrule
    
CNNSpot~\shortcite{wang2020cnn}  &70.83\phantom{\smallimprove{0.00}}	&64.43\phantom{\smallimprove{00.00}}	&63.18\phantom{\smallimprove{00.00}}	&70.67\phantom{\smallimprove{0.00}}	&69.64\phantom{\smallimprove{0.00}}	 \\
+ \textit{Ours (Sup.)} &78.22\smallimprove{7.39}	&77.61\smallimprove{13.18}	&77.40\smallimprove{14.22}	&76.79\smallimprove{6.12}	&74.46\smallimprove{4.82}	\\
+ \textit{Ours (Unsup.)}  &77.90\smallimprove{7.08}	&77.21\smallimprove{12.78}	&76.98\smallimprove{13.80}	&76.32\smallimprove{5.65}	&73.83\smallimprove{4.19}	\\
\midrule
            
FreDect~\shortcite{frank2020leveraging}  &64.03\phantom{\smallimprove{0.00}}	&69.16\phantom{\smallimprove{0.00}}	&68.10\phantom{\smallimprove{0.00}}	&65.75\phantom{\smallimprove{0.00}}	&66.53\phantom{\smallimprove{0.00}}	 \\
+ \textit{Ours (Sup.)}  &68.88\smallimprove{4.85}	&77.92\smallimprove{8.76}	&74.37\smallimprove{6.27}	&68.42\smallimprove{2.67}	&70.52\smallimprove{3.99}\\
+ \textit{Ours (Unsup.)} &65.09\smallimprove{1.06}	&76.01\smallimprove{6.85}	&72.93\smallimprove{4.83}	&65.49\smalldecrease{0.26}	&69.62\smallimprove{3.09} \\
\midrule

Fusing~\shortcite{ju2022fusing}  &68.85\phantom{\smallimprove{0.00}}		&61.82\phantom{\smallimprove{0.000}}	&61.04\phantom{\smallimprove{0.000}}	&68.08\phantom{\smallimprove{0.00}}	&66.66\phantom{\smallimprove{0.00}} \\
+ \textit{Ours (Sup.)}  &78.42\smallimprove{9.57}	&77.91\smallimprove{16.09}	&77.22\smallimprove{16.18}	&73.15\smallimprove{5.07}	&70.27\smallimprove{3.61}	 \\
+ \textit{Ours (Unsup.)} &75.60\smallimprove{6.75}	&73.79\smallimprove{11.97}	&73.53\smallimprove{12.49}	&71.82\smallimprove{3.74}	&69.02\smallimprove{2.36}	\\
\midrule 
     
LNP~\shortcite{liu2022detecting}  &81.09\phantom{\smallimprove{0.00}}	&79.37\phantom{\smallimprove{0.00}}	&79.42\phantom{\smallimprove{0.00}}	&70.21\phantom{\smallimprove{0.00}}	&69.23\phantom{\smallimprove{0.00}}	\\
+ \textit{Ours (Sup.)}  &88.42\smallimprove{7.33}	&83.40\smallimprove{4.03}	&83.47\smallimprove{4.05}	&71.77\smallimprove{1.56}	&69.88\smallimprove{0.65}\\
+ \textit{Ours (Unsup.)} &87.17\smallimprove{6.07}	&82.82\smallimprove{3.45}	&82.90\smallimprove{3.48}	&71.43\smallimprove{1.22}	&69.79\smallimprove{0.56}	\\
\midrule 

LGrad~\shortcite{tan2023learning}   &77.45\phantom{\smallimprove{0.00}}	&51.79\phantom{\smallimprove{0.00}}	&50.00\phantom{\smallimprove{0.00}}	&54.20\phantom{\smallimprove{0.00}}	&50.03\phantom{\smallimprove{0.00}}	 \\
+ \textit{Ours (Sup.)}  &79.56\smallimprove{2.11}	&61.47\smallimprove{9.68}	&56.32\smallimprove{6.32}	&61.43\smallimprove{7.23}	&49.98\smalldecrease{0.05} \\
+ \textit{Ours (Unsup.)}  &79.35\smallimprove{1.90}	&61.13\smallimprove{9.34}	&55.46\smallimprove{5.46}	&61.24\smallimprove{7.04}	&49.53\smalldecrease{0.50}	 \\
\midrule 

UnivFD~\shortcite{ojha2023towards}  &78.43\phantom{\smallimprove{0.00}}	&74.10\phantom{\smallimprove{0.00}}	&71.65\phantom{\smallimprove{0.00}}	&70.31\phantom{\smallimprove{0.00}}	&65.66\phantom{\smallimprove{0.00}}	 \\
+ \textit{Ours (Sup.)}   &86.35\smallimprove{7.92}	&82.17\smallimprove{8.07}	&80.78\smallimprove{9.13}	&77.38\smallimprove{7.07}	&70.04\smallimprove{4.38} \\
+ \textit{Ours (Unsup.)}  &86.20\smallimprove{7.77}	&82.08\smallimprove{7.98}	&80.59\smallimprove{8.94}	&77.23\smallimprove{6.92}	&69.59\smallimprove{3.93}\\
\midrule 

RINE~\shortcite{koutlis2024leveraging} &86.77\phantom{\smallimprove{0.00}}	&78.94\phantom{\smallimprove{0.000}}	&76.19\phantom{\smallimprove{0.000}}	&81.33\phantom{\smallimprove{0.00}}	&73.81\phantom{\smallimprove{0.00}} \\
+ \textit{Ours (Sup.)}  &94.56\smallimprove{7.80}	&90.93\smallimprove{11.99}	&89.19\smallimprove{13.00}	&86.17\smallimprove{4.84}	&76.53\smallimprove{2.72} \\
+ \textit{Ours (Unsup.)}  &93.79\smallimprove{7.02}	&90.34\smallimprove{11.40}	&88.63\smallimprove{12.44}	&85.61\smallimprove{4.28}	&75.54\smallimprove{1.73}	\\
\midrule
   
AIDE~\shortcite{yan2025aide} 
&92.80\phantom{\smallimprove{0.00}}	&65.91\phantom{\smallimprove{0.000}}	&60.22\phantom{\smallimprove{0.000}}	&79.48\phantom{\smallimprove{0.00}}	&68.43\phantom{\smallimprove{0.00}} \\
+ \textit{Ours (Sup.)}  &94.42\smallimprove{1.62}	&79.41\smallimprove{13.50}	&75.61\smallimprove{15.39}	&82.49\smallimprove{3.01}	&73.04\smallimprove{4.61} \\
+ \textit{Ours (Unsup.)}  &93.69\smallimprove{0.90}	&77.90\smallimprove{11.99}	&73.84\smallimprove{13.62}	&81.83\smallimprove{2.35}	&70.30\smallimprove{1.87} \\
\midrule

Effort~\shortcite{yan2024effort}  &79.35\phantom{\smallimprove{0.000}}	&67.37\phantom{\smallimprove{0.000}}	&63.56\phantom{\smallimprove{0.000}}	&75.30\phantom{\smallimprove{0.00}}	&62.52\phantom{\smallimprove{0.00}} \\
+ \textit{Ours (Sup.)}  &89.48\smallimprove{10.13}	&77.59\smallimprove{10.22}	&74.72\smallimprove{11.16}	&83.47\smallimprove{8.17}	&70.04\smallimprove{7.52} \\
+ \textit{Ours (Unsup.)}  &85.08\smallimprove{5.73\phantom{0}}	&75.30\smallimprove{7.93\phantom{0}}	&71.67\smallimprove{8.11\phantom{0}}	&81.11\smallimprove{5.81}	&67.89\smallimprove{5.37}\\

\bottomrule
    \end{tabular}
    }
\caption{{Robustness on JPEG compression and Gaussian blur.} We report the accuracies (\%) averaged over 16 test sets in AIGCDetectBenchmark~\cite{zhong2023patchcraft} with different quality factor (QF) and variance ($\sigma$).}
\label{tab:robustness}
\end{table}

\subsubsection{Unsupervised Calibration.}

In fully unsupervised settings where no labeled target data is available, we propose a calibration method that leverages the intrinsic structure of the logit distribution to recover an effective decision boundary. Our approach builds upon a fundamental assumption in binary classification: when the underlying data distribution is separable, the classifier logits tend to form a bimodal distribution, implicitly reflecting the presence of two latent modes corresponding to the two classes. 
An ideal decision threshold should hence lie near the valley between these modes and induce a symmetric partitioning of the distribution.

Formally, given a set of unlabelled target logits $\{z_i\}_{i=1}^n$, we begin by estimating the continuous probability density function $p(z)$ via Gaussian kernel density estimation:
\begin{align}
    p(z) = \frac{1}{n h} \sum_{i=1}^n K\left(\frac{z - z_i}{h}\right)\,,
\end{align}
where $K(\cdot)$ is the Gaussian kernel and $h$ is the bandwidth, whose details are introduced in Appendix.

Our goal is to find a calibration value $\alpha$ that aligns with the distributional symmetry of $p(z)$. To do this, we introduce a symmetry surrogate objective by treating $\alpha$ as a reference point and minimizing the first-order weighted shift of the distribution with respect to $\alpha$: 
\begin{align}
    \Phi(\alpha) = \int_{-\infty}^\infty (z - \alpha) \cdot p(z) \, dz\,.
\end{align}
This expression quantifies the imbalance of the logit distribution with respect to $\alpha$ as a reflective center. Intuitively, when $\Phi(\alpha) = 0$, the distribution $p(z)$ is symmetric around $\alpha$, which serves as a natural partition boundary.
Notably, we propose a moment-balancing optimization method to ensure robust unsupervised performance across distributions by adaptively penalizing low-confident logits, whose details are deferred to Appendix.

This method requires no label supervision and is robust to mild distributional shift, as it operates directly on the empirical logit geometry. It is particularly effective when the logits exhibit a latent bimodal structure, enabling recovery of semantically meaningful decision thresholds even under unsupervised conditions.

\section{Experiment}\label{sec:exp}

\subsection{Experimental Setup}

\subsubsection{Baselines.}
We evaluate nine representative off-the-shelf AI-generated image (AIGI) detectors, including CNNSpot~\cite{wang2020cnn}, FreDect~\cite{frank2020leveraging}, Fusing~\cite{ju2022fusing}, LNP~\cite{liu2022detecting}, LGrad~\cite{tan2023learning}, UnivFD~\cite{ojha2023towards}, RINE~\cite{koutlis2024leveraging}, AIDE~\cite{yan2025aide}, and
Effort~\cite{yan2024effort}.
Rather than directly comparing with these methods, we apply our calibration strategy on top of them to assess its applicability and effectiveness in enhancing detection performance.

\subsubsection{Benchmarks.}
We use two popular benchmarks in our main experiments.
(1) \textit{AIGCDetectBenchmark}~\cite{zhong2023patchcraft} covers 16 generative models, including GANs and text-to-image models such as Midjourney and DALL·E 2. 
The training set contains real images from LSUN and fake images generated by ProGAN, while the test set includes diverse fake images from the 16 generative models. 
Likewise, (2) \textit{GenImage}~\cite{zhu2023genimage} comprises ImageNet's 1,000 classes generated using 8 SOTA generators such as Stable Diffusion (SD) and Midjourney, and its training set contains real images from ImageNet and fake images generated by SD v1.4.
Benchmark details are in Appendix, 
where we also conduct experiments on other recent benchmarks.

\subsubsection{Implementation Details.}
We randomly sample a small subset of the test data (referred to as \textit{validation set}) to optimize the scalar calibration parameter~$\alpha$.
By default, we use 100 images, which is approximately 1\% of the entire test set.
The optimized $\alpha$ is then applied to the test data for evaluation.
All experiments are conducted on a single NVIDIA RTX A6000 GPU.
For each method, we use the official pretrained checkpoints without any fine-tuning or modification. If no checkpoint is available, we train the model using the official codebase.
To ensure statistical robustness, all results are averaged over 10 independent runs, and we report both the mean accuracies and standard deviations.

\subsection{Results and Analysis}\label{sec:results}

\subsubsection{Results on Standard Benchmarks.}

\cref{tab:aigcdetect,tab:genimage} illustrate the performance of representative AIGI detectors on the AIGCDetectBenchmark and GenImage datasets, respectively. Our proposed post-hoc calibration strategy consistently improves detection accuracies across different baselines. In particular, especially when the baseline models are equipped with strong pre-trained feature extractors such as CLIP (\eg, RINE~\cite{koutlis2024leveraging} and Effort~\cite{yan2024effort}), our method is able to further enhance their performance by larger margins, indicating its complementary benefit. Overall, these results demonstrate that our calibration approach can effectively unlock the latent potential of existing detectors, regardless of the backbone representation quality, and is especially beneficial in handling domain shift and subtle generation artifacts commonly present in AIGI datasets.

\subsubsection{Robustness to Image Perturbations.}

\cref{tab:robustness} shows the results under different types of image perturbations, evaluating the robustness of AIGI detectors in real-world applications under potential unseen perturbations.
Notably, our method shows strong resistance to image perturbations, especially for JPEG compression, significantly improving the performance by large margins, \eg, +15.39\% accuracy for AIDE under JPEG compression ($\text{QF}=90$). 
These results further validate the robustness and applicability of our method in real-world scenarios.

\begin{figure}[t]
  \centering
  \includegraphics[width=0.998\linewidth]{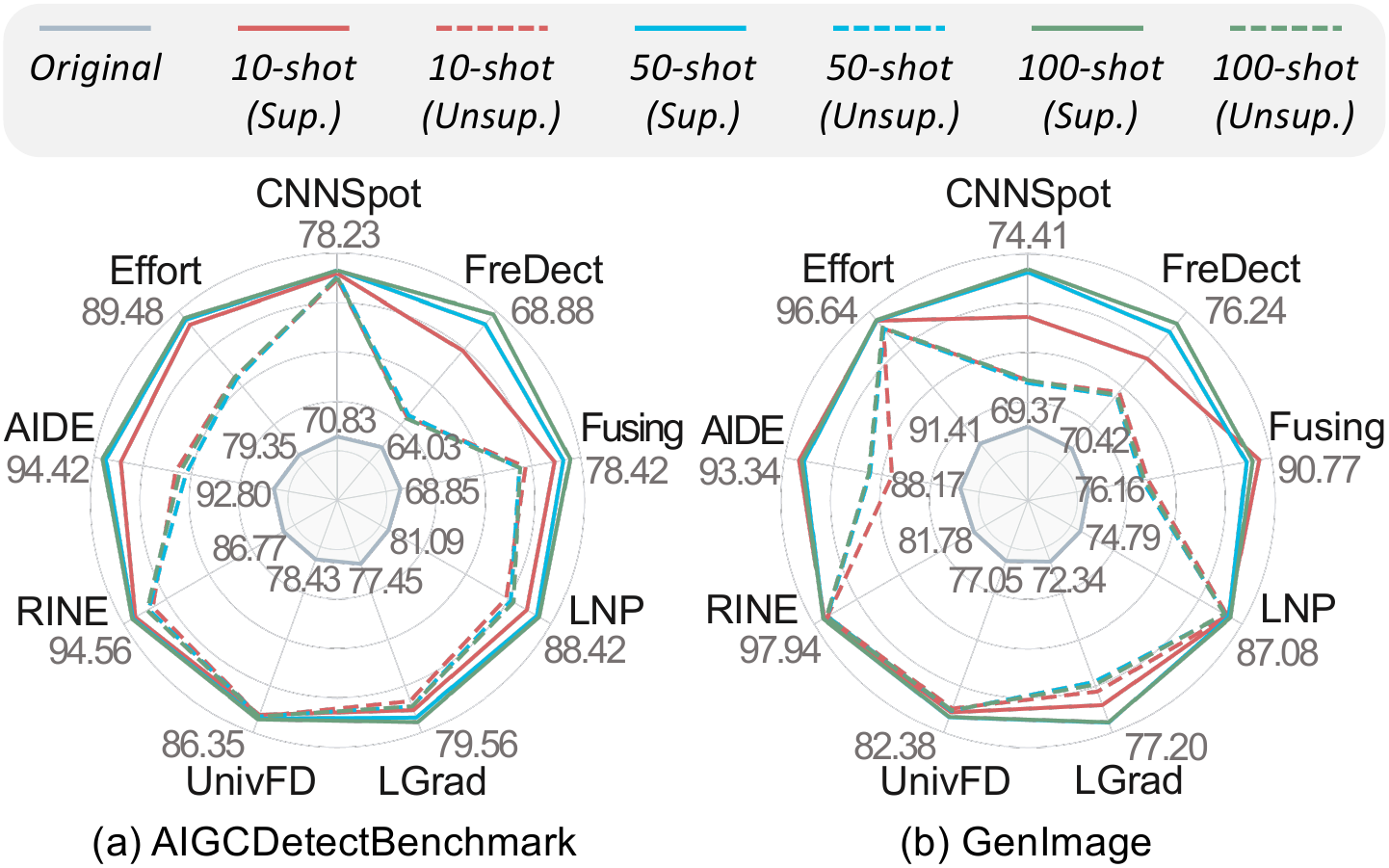}
  \caption{{Effect of validation set size on the proposed supervised and unsupervised calibration methods.
  We report the average accuracies on the two benchmarks.}
  }
  \label{fig:shots}
\end{figure}

\subsubsection{Effect of Validation Set Size.}

\cref{fig:shots} shows the performance using different validation set sizes.
We can observe that both our supervised and unsupervised calibration methods perform stably with varying amounts of validation data, demonstrating strong performance with as few as 10 samples, which is less than 0.1\% of the data in each test set.
These results further confirm the effectiveness of our method, demonstrating its utility as a lightweight and practical test-time enhancement for AIGI detection.

\subsubsection{Comparison with Different Estimation Methods for $\alpha$.}\label{sec:different_methods}

\cref{fig:different_methods} further showcases the performance using different supervised searching methods to estimate $\alpha$ in place of our proposed KDE-based method, such as binary search and training with binary cross-entropy (BCE) loss, whose details are deferred to Appendix.
We can observe that our proposed method shows strong performance even with 4 target samples in the validation set, and consistently outperforms the other two counterparts by large margins with more target samples.
These results demonstrate the data efficiency of our proposed calibration method, highlighting its practicality and suitability for real-world deployment.

\begin{figure}[t]
  \centering
  \includegraphics[width=0.998\linewidth]{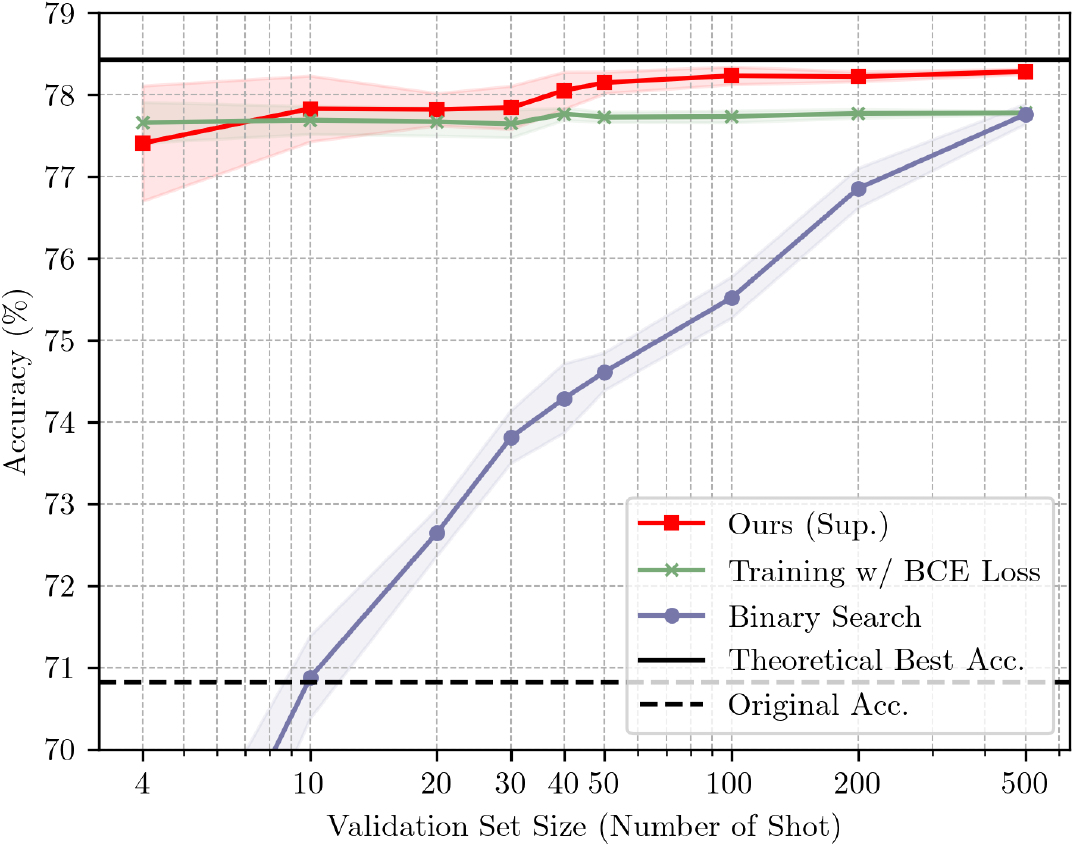}
  \caption{Performance comparison of different supervised calibration methods for estimating $\alpha$. 
  Average accuracies of CNNSpot~\cite{wang2020cnn} on AIGCDetectBenchmark~\cite{zhong2023patchcraft} are reported.
  }
  \label{fig:different_methods}
\end{figure}
\section{Conclusion}

In this work, we investigated a critical yet underexplored challenge in AI-generated image detection: the systematic misclassification of fake images under distribution shift, even in class-balanced settings. Through empirical and theoretical analyses, we identified that this phenomenon arises from a mismatch between the model's learned decision threshold and the shifted test-time distribution, particularly due to label prior and class-conditional input shifts in the fake class. To address this, we introduced a lightweight and principled post-hoc calibration method that applies a learnable scalar correction to the model's logits. Without modifying the detector's backbone, our method significantly improves detection robustness across diverse generators. Our findings underscore the importance of aligning the decision boundary with the test distribution, and demonstrate that substantial performance gains can be achieved through minimal yet targeted calibration. This highlights the untapped potential of existing detectors and opens new directions for reliable and adaptive AI-generated image detection under realistic deployment conditions.

\section*{Acknowledgments}
This research work is supported by the Agency for Science, Technology and Research (A*STAR) under its MTC Programmatic Funds (Grant No.~M23L7b0021),
and the National Natural Science Foundation of China under Grant 62571412.
The authors sincerely thank the anonymous reviewers for their valuable comments.

\bibliography{sections/bib}


\renewcommand{\thesection}{\Alph{section}}
\renewcommand{\thetable}{\Alph{table}}
\renewcommand{\thefigure}{\Alph{figure}}
\renewcommand{\theequation}{\Alph{equation}}

\setcounter{table}{0}
\setcounter{figure}{0}
\setcounter{equation}{0}

\newpage
\appendix
{

\setcounter{tocdepth}{2} 
\tableofcontents


  
  
}
\vspace{3.6pt}

\section{Experimental Details}\label{sec:experiment_details}

\subsection{Benchmark Details}\label{sec:bechmark_details}

\cref{table:dataset_aigcdetect,table:dataset_genimage} present detailed information on the two benchmarks used in our study: AIGCDetectBenchmark~\cite{zhong2023patchcraft} and GenImage~\cite{zhu2023genimage}.

\textit{AIGCDetectBenchmark}: This benchmark is trained on ProGAN and evaluated on 16 diverse test sets, covering both GAN-generated and Stable Diffusion-generated images.

\textit{GenImage}: This benchmark is trained on Stable Diffusion V1.4 and tested primarily on data generated by Stable Diffusion, with only a small subset of GAN-generated images. Notably, the Stable Diffusion-related test sets in AIGCDetectBenchmark are consistent with those used in GenImage.

\subsection{AI-Generated Image Detector Details}\label{sec:method_details}
We evaluate a range of representative AI-generated image detection methods as baselines, encompassing diverse detection paradigms.

\begin{table}[t]
\resizebox{\hsize}{!}{
\renewcommand{\arraystretch}{1.1}
\begin{tabular}{@{}l lcrl@{}}
\toprule
& {Generator}    & {Image Size} & {Number} & {Source}                                   \\ \midrule
{Train}                  & ProGAN \shortcite{karras2017progressive}   & $256 \times 256$    & 360.0k          & LSUN \shortcite{yu2015lsun}           \\ \midrule
\multirow{16}{*}{{Test}} & ProGAN \shortcite{karras2017progressive}   & $256 \times 256$    & 8.0k            & LSUN \shortcite{yu2015lsun}           \\
& StyleGAN \shortcite{karras2019style}       & $256 \times 256$    & 12.0k           & LSUN \shortcite{yu2015lsun}           \\
& BigGAN \shortcite{brock2018large}          & $256 \times 256$    & 4.0k            & ImageNet \shortcite{deng2009imagenet} \\
& CycleGAN \shortcite{zhu2017unpaired}       & $256 \times 256$    & 2.6k            & ImageNet \shortcite{deng2009imagenet} \\
& StarGAN \shortcite{choi2018stargan}        & $256 \times 256$    & 4.0k            & CelebA \shortcite{liu2015deep}        \\
& GauGAN \shortcite{park2019semantic}        & $256 \times 256$    & 10.0k           & COCO \shortcite{lin2014microsoft}     \\
& StyleGAN2 \shortcite{karras2020analyzing}  & $256 \times 256$    & 15.9k           & LSUN \shortcite{yu2015lsun}           \\
& WFIR \shortcite{WFIR}                      & $1024 \times 1024$  & 2.0k            & FFHQ \shortcite{karras2019style}      \\
& ADM \shortcite{dhariwal2021diffusion}      & $256 \times 256$    & 12.0k           & ImageNet \shortcite{deng2009imagenet} \\
& Glide \shortcite{nichol2021glide}          & $256 \times 256$    & 12.0k           & ImageNet \shortcite{deng2009imagenet} \\
& Midjourney \shortcite{midjourney}          & $1024 \times 1024$  & 12.0k           & ImageNet \shortcite{deng2009imagenet} \\
& SD v1.4 \shortcite{StableDiffusion}        & $512 \times 512$    & 12.0k           & ImageNet \shortcite{deng2009imagenet} \\
& SD v1.5 \shortcite{StableDiffusion}        & $512 \times 512$    & 16.0k           & ImageNet \shortcite{deng2009imagenet} \\
& VQDM \shortcite{gu2022vector}              & $256 \times 256$    & 12.0k           & ImageNet \shortcite{deng2009imagenet} \\
& Wukong \shortcite{wukong}                  & $512 \times 512$    & 12.0k           & ImageNet \shortcite{deng2009imagenet} \\
& DALLE 2 \shortcite{ramesh2022hierarchical} & $256 \times 256$    & 2.0k            & ImageNet \shortcite{deng2009imagenet} \\ \bottomrule
\end{tabular}

}
\caption{{Statistics of the AIGCDetectBenchmark.} ``SD'' and ``WFIR'' denote Stable Diffusion and WhichFaceIsReal, respectively. The ``Number'' column reports only the count of fake images; an equal number of real images from the same source is included for each generative model.}
\label{table:dataset_aigcdetect}
\end{table}
\begin{table}[t]

\resizebox{\hsize}{!}{
\renewcommand{\arraystretch}{1.1}
\begin{tabular}{@{}l lcrl@{}}
\toprule
& {Generator}  & {Image Size}  & {Number}  & {Source}                                                    \\ \midrule
{Train}                  & SD v1.4 \shortcite{StableDiffusion}                    & $512 \times 512$                    & 324.0k                 & ImageNet \shortcite{deng2009imagenet}                  \\ \midrule
\multirow{8}{*}{{Test}} & BigGAN \shortcite{brock2018large}     & $256 \times 256$   & 12.0k & ImageNet \shortcite{deng2009imagenet} \\
& ADM \shortcite{dhariwal2021diffusion} & $256 \times 256$   & 12.0k & ImageNet \shortcite{deng2009imagenet} \\
& Glide \shortcite{nichol2021glide}     & $256 \times 256$   & 12.0k & ImageNet \shortcite{deng2009imagenet} \\
& Midjourney \shortcite{midjourney}     & $1024 \times 1024$ & 12.0k & ImageNet \shortcite{deng2009imagenet} \\
& SD v1.4 \shortcite{StableDiffusion}   & $512 \times 512$   & 12.0k & ImageNet \shortcite{deng2009imagenet} \\
& SD v1.5 \shortcite{StableDiffusion}   & $512 \times 512$   & 16.0k & ImageNet \shortcite{deng2009imagenet} \\
& VQDM \shortcite{gu2022vector}         & $256 \times 256$   & 12.0k & ImageNet \shortcite{deng2009imagenet} \\
& Wukong \shortcite{wukong}             & $512 \times 512$   & 12.0k & ImageNet \shortcite{deng2009imagenet} \\ \bottomrule
\end{tabular}
}
\caption{{Statistics of the GenImage Benchmark.} ``SD'' stands for Stable Diffusion. 
The ``Number'' column reports only the count of
fake images; an equal number of real images from the same
source is included for each generator.
}
\label{table:dataset_genimage}
\end{table}

\begin{itemize}

    \item \textbf{CNNSpot}~\cite{wang2020cnn} demonstrates that a standard CNN classifier, when trained with basic data augmentations (e.g., JPEG compression and Gaussian blur), can generalize effectively to images generated by unseen GAN models.

    \item \textbf{FreDect}~\cite{frank2020leveraging} leverages frequency domain artifacts that are characteristic of GAN-generated images and builds a detector based on these spectral inconsistencies.

    \item \textbf{Fusing}~\cite{ju2022fusing} proposes a dual-branch architecture that integrates global spatial context with local discriminative features to enhance detection robustness.

    \item \textbf{LNP}~\cite{liu2022detecting} designs a learnable denoising network to extract noise patterns from images, which are then used as features for training a forgery classifier.

    \item \textbf{LGrad}~\cite{tan2023learning} utilizes the gradients of a pretrained CNN model as generalized representations to capture subtle artifacts introduced during image generation.

    \item \textbf{UnivFD}~\cite{ojha2023towards} adopts CLIP-derived features and trains a binary linear classifier, aiming to detect forgeries based on high-level semantic cues.

    \item \textbf{RINE}~\cite{koutlis2024leveraging} proposes a lightweight method that leverages intermediate features from CLIP and a learnable module to construct a forgery-aware representation space.
    \item \textbf{AIDE}~\cite{yan2025aide} introduces a challenging benchmark (Chameleon) to reassess the limits of AI-generated image detection and proposes AIDE, a hybrid-feature detector combining semantic and artifact-level cues.
    \item \textbf{Effort}~\cite{yan2024effort} separates forgery and semantic features via SVD to improve generalization in AI-generated image detection.

\end{itemize}

\subsection{Implementation Details}\label{sec:implementation_details}

We provide detailed optimization procedures for our proposed supervised and unsupervised calibration methods in Sec.~4. Additionally, we include the implementation details of the two supervised baselines used for comparison in Fig.~4, namely, \textit{training with BCE loss} and \textit{binary search}.

\subsubsection{Optimization for Supervised Calibration.}

To solve the supervised calibration problem in Eq.~(11), we apply a scalar bounded optimization routine (\eg, Brent’s method) over the observed logit range. 
This can be simply achieved by minimizing the classification error using common optimization tools (\eg, \texttt{minimize\_scalar} method from \texttt{scipy.optimize}):
\begin{align}
    \alpha^\star = \argmin_\alpha \left(\int_{-\infty}^\alpha p_1(z) \, dz - \int_{\alpha}^\infty p_0(z) \, dz\right)\,.
\end{align}
Running on a consumer-grade CPU (AMD EPYC 7313 3.0GHz, single-threaded), this optimization completes in an average of 17 iterations for 100-shot calibration, taking only about 0.00086 seconds, demonstrating excellent efficiency.

\subsubsection{Optimization for Unsupervised Calibration.}

Now we discuss how to estimate $\alpha^\star = \arg\min_\alpha |\Phi(\alpha)|$ in Eq.~(13).
In fact, the distributional symmetry principle introduced in Eq.~(13) admits a closed-form solution under mild smoothness assumptions. Specifically, by interpreting the threshold estimation as a moment equilibrium problem over the estimated logit density, we can directly derive an analytic expression for the optimal threshold.

To recap, let $p(z)$ denote the probability density function over the model logits $z \in \mathbb{R}$, estimated via kernel density estimation (KDE). Our goal is to find the threshold $\alpha$ such that the \textit{first moment} of the shifted density $p(z)$ centered at $\alpha$ vanishes:
\begin{align}
\Phi(\alpha) = \int (z - \alpha) \cdot p(z) \, dz = 0\,.
\end{align}
Rewriting and solving this constraint yields:
\begin{align}
\alpha^\star = \int z \cdot p(z) \, dz = \mathbb{E}_{z \sim p}[z]\,,
\end{align}
which states that the optimal threshold is simply the expected logit under the estimated distribution.

In practice, given a discretized support $\{z_j\}$ and the corresponding estimated KDE values $\{p_j = \hat{p}(z_j)\}$, we compute the threshold as
\begin{align}
\alpha^\star = \frac{\sum_j p_j \cdot z_j}{\sum_j p_j}\,.
\end{align}
This closed-form expression avoids any iterative optimization and can be interpreted as the \emph{center of mass} of the logit distribution reflected across the origin. It provides a fast and fully unsupervised mechanism to obtain a decision threshold, and empirically performs competitively when the logits exhibit a latent bimodal or skewed structure.
We show in experiments that our moment-balancing design ensures robust unsupervised performance across distributions by adaptively penalizing low-confident logits.

We also evaluated the optimization time, following the same procedure used for the supervised calibration method. 
In the 100-shot setting, the optimization completes in just 0.00041 seconds, demonstrating its extreme computational efficiency.

\subsubsection{Details of Training with BCE Loss.}

We also propose a supervised calibration method that learns an optimal logit offset by directly minimizing the binary cross-entropy (BCE) loss between calibrated logits and ground-truth labels. Specifically, we introduce a trainable scalar offset $\alpha$ such that the calibrated logits are given by $z - \alpha$, where $z$ denotes the original logits. The objective is to solve:
\begin{align}
\alpha^\star = \argmin_{\alpha} \, \mathcal{L}_{\mathrm{BCE}}(z - \alpha, y)\,,
\end{align}
where $y \in \{0, 1\}$ denotes the binary label and $\mathcal{L}_{\mathrm{BCE}}(\cdot, \cdot)$ is the standard sigmoid cross-entropy loss.

To provide a stable initialization for $\alpha$, we estimate the medians of the logits from each class and set the initial state as follows:
\begin{align}
\alpha_0 = -\frac{1}{2}\big(\mathrm{median}(z \mid y{=}0) + \mathrm{median}(z \mid y{=}1)\big)\,.
\end{align}
This reflects a symmetric midpoint between the two class distributions and empirically accelerates convergence.

We optimize $\alpha$ using the AdamW optimizer with a linear decay learning rate schedule. For the 100-shot setting, the entire optimization process converges in an average of 1K iterations, taking approximately 0.38 seconds on an NVIDIA RTX A6000 GPU.

\begin{table*}[t]
\centering
\resizebox{\linewidth}{!}{
\renewcommand{\arraystretch}{1.0} 
\setlength\tabcolsep{4pt} 
\begin{tabular}{@{}lcccccccccc@{}}
\toprule
\phantom{Dataset}              
& {CNNSpot}~\shortcite{wang2020cnn} 
& {FreDect}~\shortcite{frank2020leveraging}
& {Fusing}~\shortcite{ju2022fusing} 
& {LNP}~\shortcite{liu2022detecting}
& {LGrad}~\shortcite{tan2023learning}
& {UnivFD}~\shortcite{ojha2023towards} 
& {RINE}~\shortcite{koutlis2024leveraging}
& {AIDE}~\shortcite{yan2025aide}
& {Effort}~\shortcite{yan2024effort}   \\ \midrule
{Original}  &57.91\phantom{\err{0.00}}   & 57.90\phantom{\err{0.00}}  & 58.03\phantom{\err{0.00}}  &  57.32\phantom{\err{0.00}}  &  58.09\phantom{\err{0.00}} & 59.11\phantom{\err{0.00}}   & 58.54\phantom{\err{0.00}} &  63.06\phantom{\err{0.00}}  &   63.15\phantom{\err{0.00}}  \\

+ Ours \textit{(Sup.)}    & 67.15\err{1.15}   & 63.52\err{2.17}   & 59.18\err{0.43} & 59.10\err{0.50}  & 59.54\err{1.43} & 67.28\err{0.87}  & 58.83\err{2.07} & 67.08\err{0.17} & 67.38\err{0.89} \\
+ Ours \textit{(Unsup.)}  & 67.17\err{0.40}   & 63.48\err{0.36}   & 59.16\err{0.19} & 51.94\err{1.26}  & 54.05\err{0.97} & 67.44\err{0.49}  & 57.03\err{0.43} & 67.27\err{0.16} & 67.99\err{0.14} \\

\bottomrule
\end{tabular}
}
\caption{{Comparison on the {{Chameleon}}~\cite{yan2025aide} benchmark.} 
We report the accuracies (\%) of various detectors (columns) in distinguishing real from fake images, showing both their original performance and the performance after applying our method. All detectors are trained on real images from ImageNet and fake images generated using SD v1.4.
}
\label{tab:chameleon}
\end{table*}

\subsubsection{Details of Binary Search.}

Given a set of binary labels $\{y_i\}_{i=1}^n \in \{0,1\}$ and their corresponding logits $\{z_i\}_{i=1}^n \in \mathbb{R}$, the goal is to find an optimal classification threshold $\alpha^\star$ such that the classification accuracy is maximized. That is,
\begin{align}
\alpha^\star = \argmax_\alpha\left\{ \frac{1}{n} \sum_{i=1}^n \mathds{1} \left( z_i - \alpha > 0 \right) = y_i \right\}\,,
\end{align}
where $\mathds{1}(\cdot)$ is the indicator function.

To efficiently search for $\alpha^\star$, a binary search algorithm is employed within a predefined interval $[L, R]$, initialized by the minimum and maximum values of the logits. In each iteration, we evaluate the accuracy at both endpoints and recursively shrink the interval towards the region yielding higher classification accuracy, until the interval width falls below a small tolerance $\varepsilon$:
\begin{equation}
\begin{aligned}
    \begin{cases}
   R \leftarrow \frac{L+R}{2}\,, & \text{if } \mathrm{Acc}(\alpha_L) > \mathrm{Acc}(\alpha_R)\,; \\
    L \leftarrow \frac{L+R}{2}\,, & \text{otherwise}\,.
    \end{cases}
\end{aligned}
\end{equation}
The final estimate is $\alpha^\star \approx \frac{L+R}{2}$. 
For the 100-shot setting, the entire optimization process converges in an average of 7 iterations, taking approximately 0.0043 seconds on the aforementioned CPU device.

\begin{figure*}[ht!]
    \centering
    \includegraphics[width=0.77\linewidth]{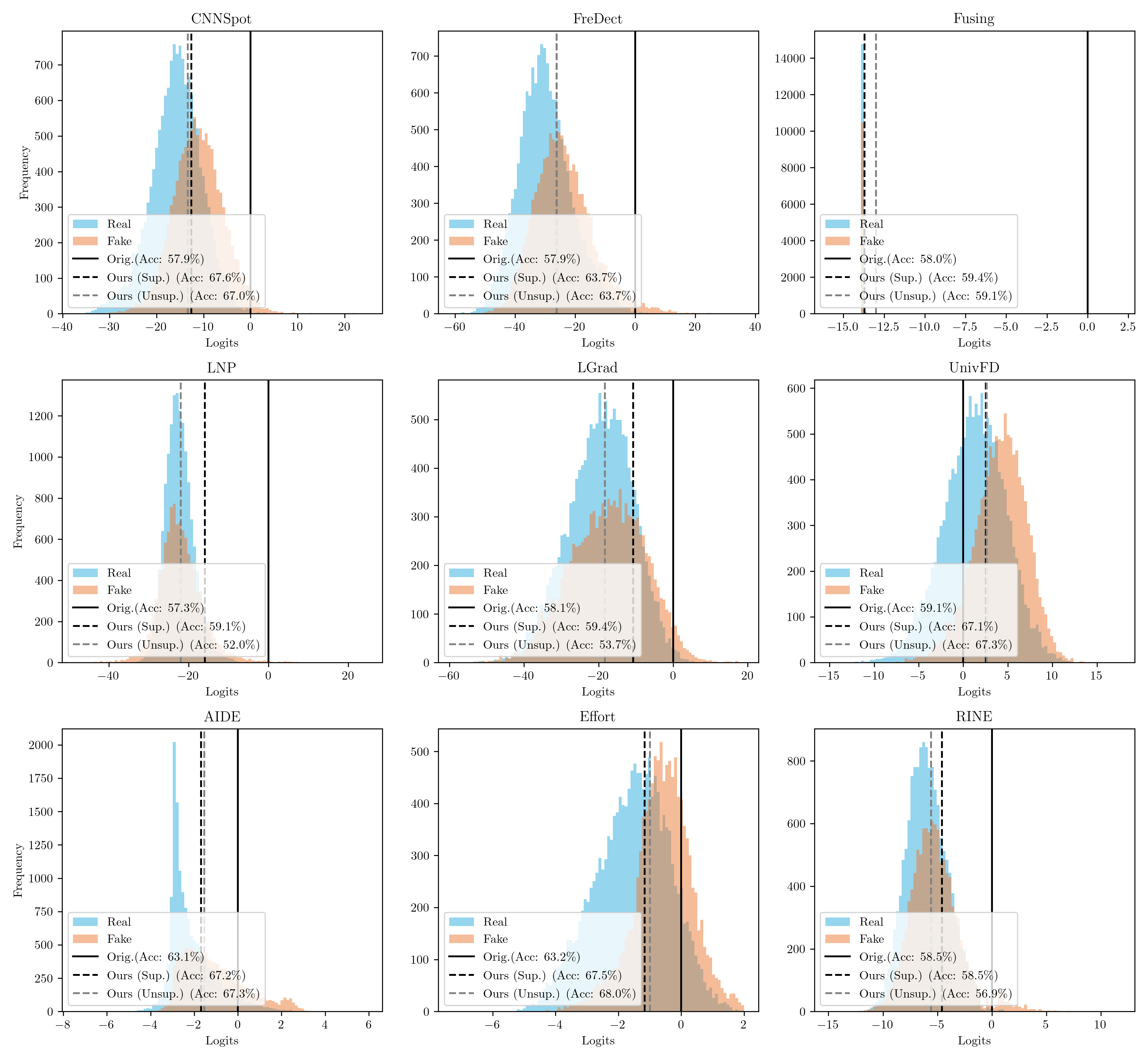}
    \caption{
    Logit distributions of the nine AI-generated image detectors on the Chameleon~\cite{yan2025aide} benchmark.
    All detectors are trained on real images from ImageNet and fake images generated using SD v1.4.
    We present both the original decision threshold and the adjusted threshold obtained through our calibration method.
    }
    \label{fig:chameleon}
\end{figure*}

\section{Additional Results and Analysis}\label{sec:additional_result}

\subsection{Results on Chameleon Benchmark}\label{sec:result_chameleon}

We also evaluate our method on the recently proposed Chameleon benchmark~\cite{yan2025aide}, which includes 11K high-fidelity AI-generated images alongside nearly 15K real-world photographs. Notably, all fake images in Chameleon have successfully passed a human ``Turing Test'', meaning human annotators consistently misidentified them as real. This makes Chameleon an exceptionally challenging benchmark that reflects the current frontier of image generation techniques.

As shown in \cref{tab:chameleon}, our proposed calibration method performs well in most cases, achieving up to a 10\% improvement in accuracy. However, a few failure cases remain after applying the unsupervised calibration.
To better understand this effect, we visualize the logit distributions produced by each AI-generated image detector, along with the updated decision thresholds resulting from our supervised and unsupervised calibration methods, respectively.
As illustrated in \cref{fig:chameleon}, the real and fake logit distributions produced by all nine detectors exhibit significant overlap. In these scenarios, simple scalar threshold adjustments may no longer suffice, underscoring the need for more robust detection approaches that can produce highly discriminative features. 
Visualizations also show that domain shifts differ across detectors trained on the same data, suggesting shifts are shaped by both data and model.
Addressing these challenges and adapting to the fast-evolving landscape of AI-generated content remains an important direction for future work.

\subsection{Additional Qualitative Results}\label{sec:additional_qualitative}

\begin{figure*}[t!]
    \centering
    \includegraphics[width=0.99\linewidth]{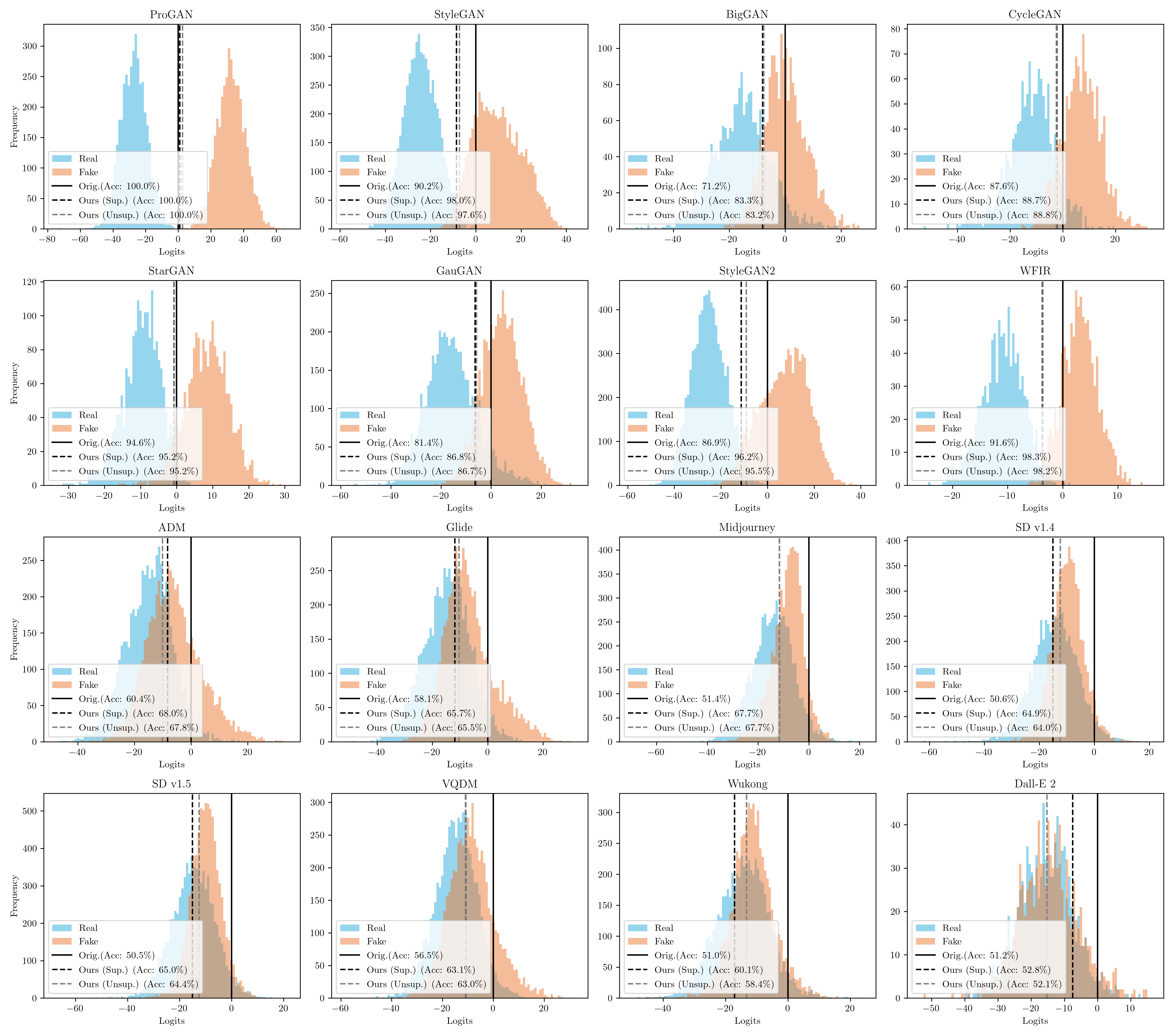}
    \caption{
    Logit distributions of the CNNSpot~\cite{wang2020cnn} detector, pretrained on ProGAN-generated fakes and evaluated on previously unseen fake images from all sixteen test sets in AIGCDetectBenchmark~\cite{zhong2023patchcraft}.
    We present both the original decision threshold and the adjusted threshold obtained through our calibration method.
    }
    \label{fig:cnnspot_full}
\end{figure*}

\begin{figure*}[t!]
    \centering
    \vspace{-18pt}
    \includegraphics[width=0.91\linewidth]{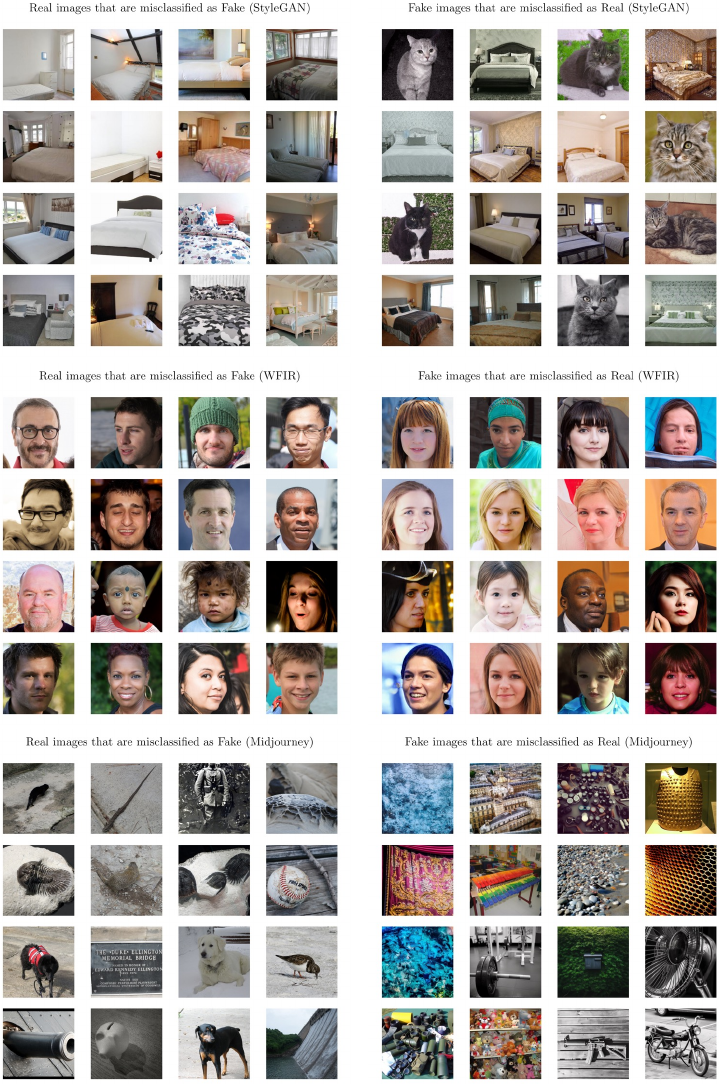}
    \caption{
    Examples of the most ambiguous real and fake images, identified as those with the highest likelihood of being misclassified. These results are obtained using CNNSpot~\cite{wang2020cnn} pretrained on ProGAN-generated fake images.
    }
    \label{fig:misclassified}
\end{figure*}

In \cref{fig:cnnspot_full}, we present the logit distributions across all 16 test sets from AIGCDetectBenchmark~\cite{zhong2023patchcraft}. Except for ProGAN images (seen during training), the detector frequently misclassifies fake images as real across the other 15 unseen test sets. Our proposed method effectively addresses this issue by calibrating the decision threshold with minimal computational overhead. Notably, the detector used in this illustration was trained on GAN-generated images and represents an older model from five years ago, resulting in suboptimal logit separation for diffusion-generated images, where real and fake logits heavily overlap. This observation underscores the growing gap between advancements in AI-generated content (AIGC) and the corresponding progress in AIGC detection. As a promising direction for future work, we aim to develop more robust detection methods that yield highly discriminative and generalizable features across diverse, unseen fake image sources.

\cref{fig:misclassified} presents some of the most ambiguous real and fake images that are frequently misclassified by the AI-generated image detector. Several patterns emerge from these examples. Misclassified fake images often exhibit higher visual complexity, with intricate patterns and rich backgrounds---characteristics commonly seen in Midjourney-generated images. In contrast, real images that are easily misclassified tend to be simpler or visually incoherent, such as the minimalistic scenes from StyleGAN or the unusual facial features in WFIR. These observations offer valuable insight into how AI-generated image detectors distinguish between real and fake images, potentially guiding the design of more robust models in future research.

\end{document}